\documentclass{article}
\pdfoutput=1
\usepackage[margin=28mm]{geometry}
\usepackage{amsfonts}
\usepackage{algorithmic}
\usepackage{algorithm}
\usepackage{amsmath}
\usepackage{amssymb}
\usepackage{multirow}
\usepackage{graphicx}
\usepackage{natbib}
\usepackage{url}
\usepackage{tikz}
\usepackage{paralist}

\def\BIC{\textsc{bic} }
\def\PC{\textsc{pc} }

\def\DAG{\textsc{dag} }
\def\CPDAG{\textsc{cpdag} }
\def\GES{\textsc{ges} }

\def\SEMSEV{$\textsc{gds}_{\textsc{eev}}$ }

\def\SEMs{structural equation models }

\usetikzlibrary{arrows}

\newdimen\arrowsize
\pgfarrowsdeclare{arcsq}{arcsq}
{
  \arrowsize=0.2pt
  \advance\arrowsize by .5\pgflinewidth
  \pgfarrowsleftextend{-4\arrowsize-.5\pgflinewidth}
  \pgfarrowsrightextend{.5\pgflinewidth}
}
{
  \arrowsize=1.5pt
  \advance\arrowsize by .5\pgflinewidth
  \pgfsetdash{}{0pt} 
  \pgfsetroundjoin   
  \pgfsetroundcap    
  \pgfpathmoveto{\pgfpoint{0\arrowsize}{0\arrowsize}}
  \pgfpatharc{-90}{-140}{4\arrowsize}
  \pgfusepathqstroke
  \pgfpathmoveto{\pgfpointorigin}
  \pgfpatharc{90}{140}{4\arrowsize}
  \pgfusepathqstroke
}

\newcommand{\independent}{\mbox{${}\perp\mkern-11mu\perp{}$}}
\newcommand{\notindependent}{\mbox{${}\not\!\perp\mkern-11mu\perp{}$}}

\DeclareMathOperator*{\argmin}{argmin}

\newcommand{\Cov}[1]{\mathrm{Cov}({#1})}
\newcommand{\C}[1]{\mathcal{#1}}
\newcommand{\law}[1]{\mathcal{L}({#1})}
\newcommand{\lawX}{{\law{\mathbf X}}}
\newcommand{\lawN}{{\law{\mathbf N}}}
\newcommand{\B}[1]{\mathbf{#1}}

\newcommand{\R}{{\mathbb R}}

\newcommand{\N}{{\mathbf N}}
\newcommand{\X}{{\mathbf X}}

\newcommand{\Gp}{\C{G}'}
\newcommand{\G}{\C{G}}
\newcommand{\SE}{\C{S}}

\newcommand{\var}[1]{{\mathrm{var}(#1)}}

\newcommand{\PA}[2][]{{\B{PA}}^{#1}_{#2}}
\newcommand{\CH}[2][]{{\B{CH}}^{#1}_{#2}}
\newcommand{\DE}[2][]{{\B{DE}}^{#1}_{#2}}
\newcommand{\ND}[2][]{{\B{ND}}^{#1}_{#2}}
\newcommand{\eref}[1]{(\ref{#1})}
\newcommand{\given}{\,\mid\,}

\newenvironment{proof}[1][\!. ]{{\bf Proof #1}}{\hfill$\square$\vskip\baselineskip}

\newtheorem{theorem}{Theorem}

\newtheorem{lemma}[theorem]{Lemma}

\newtheorem{remark}[theorem]{Problem}

\begin{document}

\title{Identifiability of Gaussian structural equation models with equal error variances}

\author{
Jonas Peters\thanks{\{peters, buhlmann\}@stat.math.ethz.ch} \\
Seminar for Statistics\\
ETH Zurich\\
Switzerland
\and 
Peter B\"uhlmann$^*$ \\
Seminar for Statistics\\
ETH Zurich\\
Switzerland
}


\maketitle

\begin{abstract}
We consider structural equation models in which variables can be written as a function of their parents and noise terms, which are assumed to be jointly independent. Corresponding to each structural equation model, there is a directed acyclic graph describing the relationships between the variables. In Gaussian structural equation models with linear functions, the graph can be identified from the joint distribution only up to Markov equivalence classes, assuming faithfulness. In this work, we prove full identifiability if all noise variables have the same variances: the directed acyclic graph can be recovered from the joint Gaussian distribution. Our result has direct implications for causal inference: if the data follow a Gaussian structural equation model with equal error variances and assuming that all variables are observed, the causal structure can be inferred from observational data only. We propose a statistical method and an algorithm that exploit our theoretical findings.
\end{abstract}

\section{Introduction}
\subsection{Graphical and structural equation models}
For random variables $X_1, \ldots, X_p$, we define a graphical model as a pair $\{\G, \lawX \}$, where $\lawX = \law{X_1, \ldots, X_p}$ is a joint probability distribution that is Markov with respect to a directed acyclic graph $\G$ \citep[Chapter~3.2]{Lauritzen1996}.
Structural equation models, also referred to as a functional models, are related to graphical models. They are specified by a collection $\SE = \{S_1, \ldots, S_p\}$ of $p$ equations
\begin{equation} \label{eq:sem}
S_j: \quad X_j = f_j(X_{\PA{j}}, N_j)\quad (j=1, \ldots, p)
\end{equation}
and a joint distribution $\lawN = \law{N_1, \ldots, N_p}$ of the noise variables. Here, $\PA{j} \subset \{1, \ldots, p\} \setminus \{j\}$ denotes the parents of $j$. We require the noise terms to be jointly independent, so $\lawN$ is a product distribution. The graph $\G$ of a structural equation model is obtained by drawing directed edges from each variable $X_k$, $k \in \PA{j}$, occurring on the right-hand side of equation~\eqref{eq:sem} to $X_j$. The graph $\G$ is required to be acyclic.
Furthermore, given a structural equation model, the joint distribution $\lawX$ is fully determined and
$\lawX$ is Markov with respect to the graph $\G$ \citep[Theorem~1.4.1]{Pearl2009}.

\subsection{Identifiability from the distribution} \label{sec:ini}
We address the following problem. Given the joint distribution $\lawX = \law{X_1, \ldots, X_p}$ from a graphical model or from a structural equation model with directed acyclic graph $\G_0$, can we recover the graph $\G_0$?
By first considering graphical models one can easily see that the answer is negative: the joint distribution $\lawX$ is Markov with respect to different directed acyclic graphs, e.g., to all fully connected directed acyclic graphs. Thus, there are many possible graphical models $\{\G, \lawX \}$ for the same distribution $\lawX$. Similarly, there are structural equation models with different structures that could have generated the distribution $\lawX$.
By making additional assumptions one obtains restricted graphical models 
and 
restricted structural equation models
for which the graph is identifiable from the joint distribution. It is precisely here that the difference between graphical and functional models becomes apparent.

Given a graphical model, the distribution $\lawX$ is faithful with respect to the directed acyclic graph $\G_0$ if each conditional independence found in $\lawX$ is implied by the Markov condition. If faithfulness holds, one can obtain the Markov equivalence graph of the true directed acyclic graph $\G_0$ \citep{Spirtes2000}. 
But the Markov equivalence class may still be large \citep[cf.][]{Andersson1997} and the directed acyclic graph $\G_0$ is not identifiable. 
Furthermore, faithfulness in its full generality cannot be tested from data \citep{Zhang2008}. 
Since both the Markov condition and faithfulness only restrict the conditional independences in the joint distribution, it is not surprising that two graphs entailing the same conditional independences cannot be distinguished.

Structural equation models enable us to exploit a different type of restriction.
First, a general Gaussian structural equation model is equivalent to a Gaussian graphical model $\{\G_0, \lawX\}$, so the structure $\G_0$ is not identifiable from $\lawX$. 
Recently, however, it has been shown that this case is exceptional: (i) if we consider linear functions and non-Gaussian noise, one can identify the underlying directed acyclic graph $\G_0$ \citep{lingam}; (ii) if one restricts the functions to be additive in the noise component and excludes the linear Gaussian case, as well as a few other pathological function-noise combinations, one can show that $\G_0$ is identifiable from $\lawX$ \citep{Hoyer2008, Peters2011}. 
In this work, we prove that there is a third way to deviate from the general linear Gaussian case: (iii) Gaussian structural equation models where all functions are linear, but the normally distributed noise variables have equal variances $\sigma^2$, are again identifiable.
The identifiability results (i) and (ii) require a condition called causal minimality.
In its original form, \citet{Zhang2008} define causal minimality as follows:
for the true causal graph $\G_0$, $\lawX$ is not Markov to any proper subgraph of $\G_0$. 
Causal minimality is therefore a weak form of faithfulness. Remark~\ref{rem:cmi} shows that for proving (iii) we assume causal minimality.

It may come as a surprise that for a class of Gaussian structural equation models the underlying directed acyclic graph is identifiable. The assumption of equal error variances seems natural for applications with variables from a similar domain and is commonly used in time series models.

\subsection{Causal interpretation}
Our result has implications for causal inference. If $\G_0$ is interpreted as the causal graph of the data generating process for $X_1, \ldots, X_p$, the problem considered here is to infer the causal structure from the joint distribution. This is particularly interesting when the causal graph is of interest but interventional experiments are too expensive, unethical or even impossible to perform.
In the causal setting, our result reads as follows. If the observational data are generated by a Gaussian structural equation model that represents the causal relationships and has equal error variances, then the causal graph is identifiable from the joint distribution. 
Despite the potentially
important application in causal inference, we present the main statement
and its proof without causal terminology; in particular, equations~\eqref{eq:sem} and~\eqref{eq:mm} can be interpreted as holding in distribution.


\section{Identifiability for Gaussian models with equal error variances} \label{sec:the}
We first introduce some notation.
The index set $\B{J} =\{1, \ldots, p\}$ corresponds to a set of vertices in a graph. Associated with $j \in \B{J}$ are random variables $X_j$ from $\B{X} = (X_1, \ldots, X_p)$.
Given a directed acyclic graph $\G$, we denote the parents of a node $j$ by $\PA[\G]{j}$, the children by $\CH[\G]{j}$, the descendants by $\DE[\G]{j}$ and the non-descendants by $\ND[\G]{j}$.


We consider a structural equation model with directed acyclic graph $\G_0$ of the form 
\begin{equation} \label{eq:mm}
X_j = \sum_{k \in \PA[\G_0]{j}} \beta_{jk} X_k + N_j \quad (j=1, \ldots, p)\,,
\end{equation}
where all $N_j$ are independent and identically distributed according to $\mathcal{N}(0,\sigma^2)$ with $\sigma^2 > 0$.
Additionally, for each $j \in \{1, \ldots, p\}$ we require $\beta_{jk} \neq 0$ for all $k \in \PA[\G_0]{j}$.

\begin{theorem} \label{thm:main}
Let $\lawX$ be generated from model~\eqref{eq:mm}.
Then $\G_0$ is identifiable from $\lawX$ and the coefficients $\beta_{jk}$ can be reconstructed for all $j$ and $k \in \PA[\G_0]{j}$.
\end{theorem}

\begin{remark}
\label{rem:proofidea}
The idea of the proof is to assume that there are two structural equation models with distinct graphs $\G$ and $\Gp$ that lead to the same joint distribution. 
We exploit the Markov condition and causal minimality, see Remark~\ref{rem:cmi}, in order to find variables $L$ and $Y$ that have the same set of parents $\B{S} = \{S_1, S_2\}$ in both graphs, but reversed edges between each other in $\G$ and $\Gp$, as shown in Fig.~\ref{fig:ii2}. 
\begin{figure}[t]
  \begin{minipage}[t]{0.48\columnwidth}
   \begin{center}
    \begin{tikzpicture}[xscale=1.4, yscale=0.6, line width=0.5pt, minimum size=0.58cm, inner sep=0.3mm, shorten >=1pt, shorten <=1pt]
    \normalsize
    \draw (-1,0) node(x) [circle, draw] {$L$};
    \draw (1,0) node(y) [circle, draw] {$Y$};
    \draw (-0.7,1.7) node(s1) [circle, draw] {$S_1$};
    \draw (0.7,1.7) node(s2) [circle, draw] {$S_2$};
    \draw[-arcsq] (s1) -- (x);
    \draw[-arcsq] (s2) -- (x);
    \draw[-arcsq] (s1) -- (y);
    \draw[dashed,-arcsq] (y) -- (x);
    \draw[-arcsq] (s2) -- (y);
   \end{tikzpicture}\\
Graph $\G$ 
 \end{center}
 \end{minipage}
 \hspace{0.02\columnwidth}
 \begin{minipage}[t]{0.48\columnwidth}
  \begin{center} 
    \begin{tikzpicture}[xscale=1.4, yscale=0.6, line width=0.5pt, minimum size=0.58cm, inner sep=0.3mm, shorten >=1pt, shorten <=1pt]
    \normalsize
    \draw (-1,0) node(x) [circle, draw] {$L$};
    \draw (1,0) node(y) [circle, draw] {$Y$};
    \draw (-0.7,1.7) node(s1) [circle, draw] {$S_1$};
    \draw (0.7,1.7) node(s2) [circle, draw] {$S_2$};
    \draw[-arcsq] (s1) -- (x);
    \draw[dashed,-arcsq] (x) -- (y);
    \draw[-arcsq] (s2) -- (x);
    \draw[-arcsq] (s1) -- (y);
    \draw[-arcsq] (s2) -- (y);
   \end{tikzpicture}\\  
Graph $\Gp$
  \end{center}
\end{minipage}
\caption{The situation dealt with in the second part of case (ii) of the proof of Theorem~\ref{thm:main}, with $\B{S}=\{S_1, S_2\}$ and $\B{D}=\emptyset$. It contains the proof's main argument.}
 \label{fig:ii2}
\end{figure}
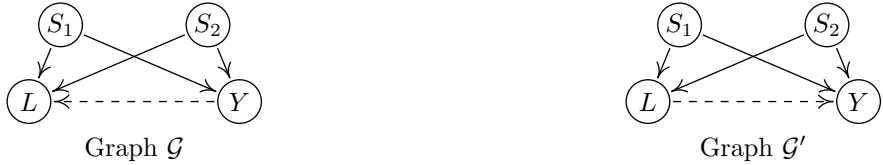
Defining $L^* = L{\given}_{\B{S} = s}$ for some value $s \in \R^2$, we can use the equal error variances to show that $L^*$ has different variances in both graphs. This leads to a contradiction.
\end{remark}

\begin{remark}
\label{rem:cmi}
Theorem~\ref{thm:main} assumes that the coefficients $\beta_{jk} \neq 0$ do not vanish for any $k \in \PA[\G_0]{j}$.
Lemma~\ref{lem:cmc} below and Proposition~2 in \citet{Peters2011} show that this condition implies causal minimality.
From our point of view, causal minimality is a natural condition and in accordance with the intuitive understanding of a causal influence between variables.



\end{remark}

\begin{remark}
\label{rem:sca}
Theorem~\ref{thm:main} can be generalized to the case where the error covariance matrix has the form
$
\Cov{N_1, \ldots, N_p} = \sigma^2  \mathrm{diag}(\alpha_1, \ldots, \alpha_p)
$
with pre-specified $\alpha_1, \ldots, \alpha_p$ and unknown $\sigma^2$. 
\end{remark}
\section{Penalized maximum likelihood estimator} \label{sec:pmle}
Consider data which are independent and identically distributed realizations of $\X^{(1)},\ldots ,\X^{(n)}$ from model \eqref{eq:mm} with true coefficients $\beta_{jk}^0$. The representation in vector form is
$
\X =B \X + \N,
$
where $B$ is the $p \times p$ matrix with entries $B_{jk} = \beta_{jk}$. To make the manuscript easier to read we write $B$ or $\beta$ whenever we think of a matrix or a vector of parameters, respectively.
As estimator for the coefficients $B^0 = (\beta^0_{jk})_{j,k}$ and the
error variance $\sigma^2$, we consider 
\begin{eqnarray}\label{est}
\{\hat{\beta}(\lambda), \hat{\sigma}^2(\lambda)\} = \argmin_{\beta \in \mathcal{B},\sigma^2 \in \R^+} -
\ell(\beta,\sigma^2;\X^{(1)},\ldots ,\X^{(n)}) + \lambda \|\beta\|_0 \,,
\end{eqnarray}
where 
$$
-\ell(\beta,\sigma^2;\X^{(1)},\ldots ,\X^{(n)}) =
\frac{np}{2}\log(2\pi \sigma^2) + \frac{n}{2\sigma^2}\mathrm{tr}\{(I-B)^T (I-B)  \hat{\Sigma}\}\,,
$$
with sample covariance matrix $\hat{\Sigma}$,
is the negative log-likelihood assuming equal error variances $\sigma^2$ and $\|\beta\|_0 = |\{j,k\,:\,\beta_{jk} \neq
0\}|$. Furthermore, 
$\mathcal{B} = \{B \in \R^{p \times p}\,:\, \mathrm{Adj}(B) \text{ has only zero eigenvalues}\}$ contains 
only those coefficient matrices whose corresponding graphs do not have cycles 
\citep[p.81]{Cvetkovic1995}.
Here, $\mathrm{Adj}(B)_{jk} = 1_{\beta_{jk} \neq 0}$ is the adjacency matrix.
Minimizing over all $\beta \in \mathcal{B}$ includes optimizing over all directed acyclic graphs, see Section~\ref{sec:gsa}.
The induced directed acyclic graph from
$\hat{\beta}(\lambda)$ is denoted by $\hat{\G}$. 
For $\lambda = \log(n)/2$ the objective function in equation~\eqref{est} is the \BIC score.

The convergence rate and consistency of the penalized maximum likelihood estimator for the true
coefficients $\beta^0_{jk}$ and the true structure ${\cal G}_0$ follow from
an analysis in
\citet[Theorem~5.1]{vandeGeer2012}, under regularity conditions. More precisely, for $\lambda_n = \log(n)/2$ we
have
\begin{eqnarray}\label{propert}
& \sum_{j,k=1}^p \{\hat{\beta}_{jk}(\lambda_n) - \beta^0_{jk}\}^2 = O_P\{\log(n) n^{-1}\} \quad
&(n \to \infty)\,,\nonumber\\
& \mathrm{pr}(\hat{\cal G}_n = {\cal G}_0) \to 1 
& (n \to \infty)\,. \nonumber
\end{eqnarray}
The results in \citet[Section~5]{vandeGeer2012} also cover the high-dimensional sparse setting
where $p = p_n = O\{n/\log(n)\}$. 

One could use a combination of the PC-algorithm and
minimization of the penalized likelihood in equation~\eqref{est}: the former, which
is computationally very efficient, could be used for estimating the Markov
equivalence class and the latter for orienting remaining undirected
edges. A related approach has been
suggested by \citet{Tillman2009}. For consistency in the first step
one necessarily requires a version of the strong faithfulness assumption, which can be very
restrictive \citep{Uhler2013}. Penalized maximum likelihood
estimation does not need such an assumption \citep{vandeGeer2012} but pays a price in terms of computational complexity.

\section{Greedy search algorithm} \label{sec:gsa}
Because the optimization in equation~\eqref{est} is over the space of all directed acyclic graphs, the estimator is hard to compute. Already for $p=20$, there are $2.3 \times 10^{72}$ directed acyclic graphs \citep{OEIS}, which makes an exhaustive search infeasible. 
Instead, we propose a greedy procedure that we call greedy directed acyclic graph search with equal error variance. At each iteration $t$ we are given a directed acyclic graph $\G_t$ and move to the neighbouring directed acyclic graph with the largest drop in the \BIC score. If all neighbours have a higher \BIC score in equation~\eqref{est} than $\G_t$, the algorithm terminates. Here, we say that two directed acyclic graphs are neighbours if they can be transformed into each other by one edge addition, removal or reversal. \citet{Chickering2002} proposes a similar search strategy but with the search done in the space of Markov equivalence classes rather than over directed acyclic graphs.

In order to shorten the runtime, we randomly search through neighbouring directed acyclic graphs until we find a directed acyclic graph with a better score than $\G_t$ and use this directed acyclic graph for $\G_{t+1}$. We consider at least $k$ neighbours; if there are several directed acyclic graphs among the first $k$ with better scores than $\G_t$, we take the best one. The whole procedure further improves if we increase the probability of changing edges pointing into nodes whose residuals have a high variance. This modification and the score function are the only parts of the algorithm that make use of the equal error variances. Additionally, we restart the method five times starting from a random sparse graph with $k=p, k=2p, k=3p, k=5p$ and $k=300$. This choice is ad hoc but works well in practice, as it decreases the risk of getting stuck in a local optimum. R code for this method is available as Supplementary Material.

\section{Experiments}
\subsection{Existing methods}
We compare our method against the PC-algorithm \citep{Spirtes2000} and greedy equivalence search \citep{Chickering2002}. The latter approximates the \textsc{BIC}-regularized maximum likelihood estimator for non-restricted Gaussian structural equation models. Both methods can only recover the Markov equivalence class, see Section~\ref{sec:ini}, and therefore leave some arrows undirected. The Markov equivalence class can be represented by a completed partially directed acyclic graph. In the experiments, we report the structural Hamming distance between the true and estimated partially directed acyclic graphs; this assigns a distance of two for each pair of reversed edges, for example, $\rightarrow$ in the true and $\leftarrow$ in the estimated graph; all other edge mistakes count as one.

\subsection{Random graphs} \label{sec:rg}
For varying $n$ and $p$ we compare the three methods.  
For a given value $p$, we randomly choose an ordering of the variables with respect to the uniform distribution and include each of the $p(p-1)/2$ possible edges with a probability of $p_{\mathrm{edge}}$.
All noise variances are set to $1$ since scaling all noise variables with a common factor yields exactly the same estimates $\hat \beta$ and $\hat \G$. The coefficients $\beta^0_{jk}$ are uniformly chosen from $[-1,-0.1] \cup [0.1,1]$. 
We consider a sparse setting with $p_{\mathrm{edge}} = 3/(2p-2)$, which results in an expected number of $3p/4$ edges, and a dense setting with $p_{\mathrm{edge}} = 0.3$. 
Table~\ref{tab:sparse} shows the average structural Hamming distance to the true directed acyclic graph and to the true completed partially directed acyclic graph over $100$ simulations for the sparse setting.
Except for $p=40$ and $n=100$, the graphs estimated by the proposed method are closer to the true directed acyclic graph than the resulting graphs from state of the art methods, who can only recover the true Markov equivalence class; greedy directed acyclic graph search also performs better when comparing the distance to the true completed partially directed acyclic graph.
Table~\ref{tab:dense} shows the analogous results for the dense setting, in which the improvement with greedy directed acyclic graph search with equal error variances is even larger.

\begin{table}[ht]
\begin{tabular}{rcccccccccc} 
\multicolumn{2}{c}{} & \multicolumn{3}{c}{$n=100$} & \multicolumn{3}{c}{$n=500$} & \multicolumn{3}{c}{$n=1000$}\\ 
$p$& & \SEMSEV & \PC & \GES & \SEMSEV & \PC & \GES& \SEMSEV & \PC & \GES \\ 
\multirow{2}{*}{$5$}  & \DAG & $1.5$ &$3.9$&$3.6$& $0.5$&$2.9$&$2.8$& $0.4$&$3.0$&$2.5$ \\ 
&                     \CPDAG & $1.5$ &$2.9$&$2.3$& $0.5$&$1.4$&$1.2$& $0.3$&$1.0$&$0.7$ \\ 
\multirow{2}{*}{$20$} & \DAG & $12.2$&$14.1$&$18.0$& $4.5$&$11.1$&$10.3$& $2.7$&$10.1$&$8.7$ \\ 
&                     \CPDAG & $13.9$&$10.9$&$17.0$& $5.2$&$7.7$&$7.6$& $3.0$&$6.9$&$5.6$ \\ 
\multirow{2}{*}{$40$} & \DAG & $44.7$&$29.6$&$53.0$& $15.7$&$22.6$&$26.1$& $10.7$&$20.1$&$21.9$ \\ 
&                     \CPDAG & $50.0$&$24.4$&$53.1$& $18.9$&$15.9$&$23.4$& $13.4$&$13.3$&$17.5$ 
\end{tabular}
\label{tab:sparse}
\caption{Structural Hamming distance between estimated and true directed acyclic graph and estimated and true Markov equivalence class, for sparse graphs with $p$ nodes and sample size $n$.
\textsc{dag}, directed acyclic graph; \textsc{cpdag}, completed partially directed acyclic graph; $\textsc{gds}_{\textsc{eev}}$, greedy directed acyclic graph search with equal error variances; \textsc{pc}, PC-algorithm; \textsc{ges}, greedy equivalence search.}
\end{table}

\begin{table}[ht]
\begin{tabular}{rcccccccccc} 
\multicolumn{2}{c}{} & \multicolumn{3}{c}{$n=100$} & \multicolumn{3}{c}{$n=500$} & \multicolumn{3}{c}{$n=1000$}\\ 
$p$& & \SEMSEV & \PC & \GES &                         \SEMSEV & \PC & \GES&    \SEMSEV & \PC & \GES \\ 
\multirow{2}{*}{$5$}  & \DAG & $1.2$&$2.9$&$3.0$&       $0.6$&$2.4$&$2.2$&       $0.3$&$2.1$&$2.1$ \\ 
&                     \CPDAG & $1.3$&$2.1$&$1.9$&       $0.5$&$1.2$&$0.7$&        $0.2$&$0.8$&$0.5$ \\ 
\multirow{2}{*}{$20$} & \DAG & $30.0$&$56.6$&$63.9$&    $12.5$&$55.7$&$66.3$&     $8.2$&$57.6$&$69.1$ \\ 
&                     \CPDAG & $31.0$&$56.1$&$63.2$&    $13.1$&$55.5$&$66.2$&     $8.8$&$57.5$&$68.5$ \\ 
\multirow{2}{*}{$40$} & \DAG & $216.1$&$242.8$&$323.1$& $185.2$&$247.2$&$430.4$&  $172.0$&$248.9$&$470.6$ \\ 
&                     \CPDAG & $217.1$&$242.4$&$323.0$& $185.7$&$247.0$&$430.1$&  $172.2$&$248.5$&$470.4$ 
\end{tabular}
\label{tab:dense}
\caption{Structural Hamming distance between estimated and true directed acyclic graph and estimated and true Markov equivalence class, for dense graphs with $p$ nodes and sample size $n$.
\textsc{dag}, directed acyclic graph; \textsc{cpdag}, completed partially directed acyclic graph; $\textsc{gds}_{\textsc{EEV}}$, greedy directed acyclic graph search with equal error variances; \textsc{pc}, PC-algorithm; \textsc{ges}, greedy equivalence search.}
\end{table}

As a proof of concept, we also simulate data with $n=500$ from a non-faithful distribution:  $X_1 = N_1$, $X_2 = - X_1 + N_2$ and $X_3 = X_1 + X_2 + N_3$. As stated by the theory, the PC-algorithm and greedy equivalent search fail here: in all $100$ experiments, they output $X_1 \rightarrow X_2 \leftarrow X_3$, which is not the correct Markov equivalence class. Greedy directed acyclic graph search always identified the correct directed acyclic graph.

\subsection{Deviation from equal error variances} \label{sec:devse} 
When the data are generated by a Gaussian structural equation model with different error variances, the method is not guaranteed to find the correct directed acyclic graph or the correct Markov equivalence class.
When the true data generating process follows such a Gaussian structural equation model with different variances, we can always represent it as a model with equal error variances if we apply a fine-tuned rescaling of the variables $X_i \mapsto a_i X_i$ with $a_i$ equal to the inverse of the standard deviation of the error in the $i$th structural equation. Of course, such a rescaling is only possible
when knowing the error variances, hence the word fine-tuned. In the hypothetical case where the data would be scaled with such a deceptive fine-tuned standardization, the graph identified by our method would belong to the correct Markov equivalence class. We emphasize, however, that this is for an artificial scenario which is different from having raw data from a Gaussian structural equation model with different error variances.
An important question is how sensitive our method is to deviations from the assumption of equal error variances.
We investigate this empirically. For $p=10$ and $n=500$, we
sample the noise variances uniformly from $[1-a,1+a]$ and vary $a$ between $0$ and $0.9$.
Theorem~\ref{thm:main} establishes identifiability of the graph only for $a = 0$. 
As before, the coefficients $\beta^0_{jk}$ are uniformly chosen from $[-1,-0.1] \cup [0.1,1]$. The parameter $p_{\mathrm{edge}}$ is chosen to be $2/(p-1)$, on average resulting in $p$ edges; this is in between the sparse and the dense setting.
Figure~\ref{fig:dev} shows that the performance of greedy directed acyclic graph search is relatively robust as the parameter $a$ changes. Even for large values of $a$, the method does not perform worse than the PC-algorithm.
The best-score method reports the result of greedy directed acyclic graph search or greedy equivalence search depending on which method obtained the better score. Greedy directed acyclic graph search was chosen in $100\%$, $100\%$, $88\%$, $36\%$, $7\%$, $1\%$, $2\%$, $0\%$, $0\%$ and $0\%$ of the cases, for $a$ ranging between $0$ and $0.9$, respectively.
\begin{figure} \label{fig:dev}
\begin{center}
\includegraphics[width=0.9\textwidth]{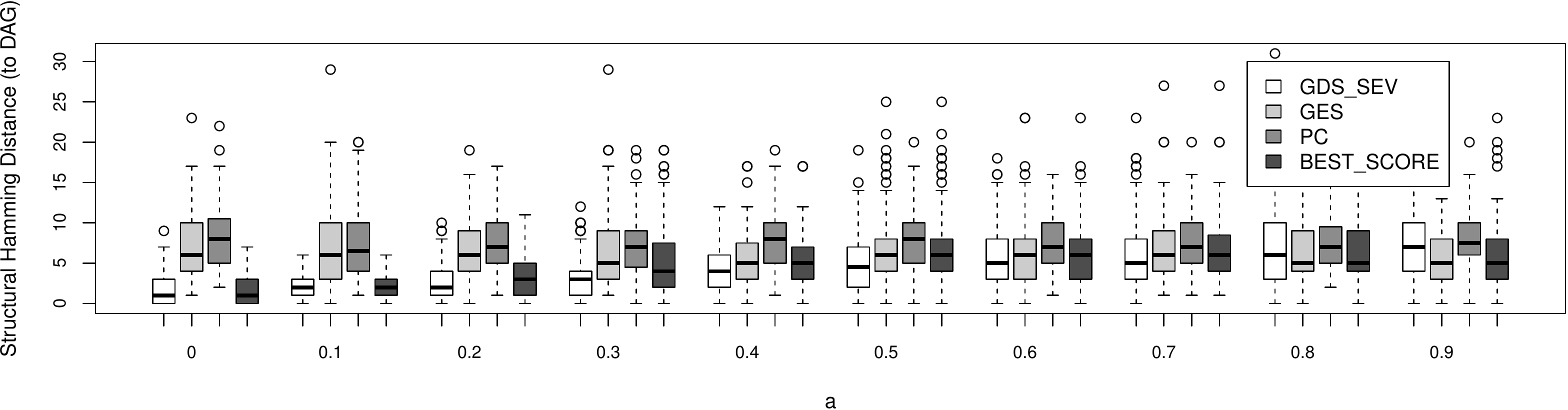}
\includegraphics[width=0.9\textwidth]{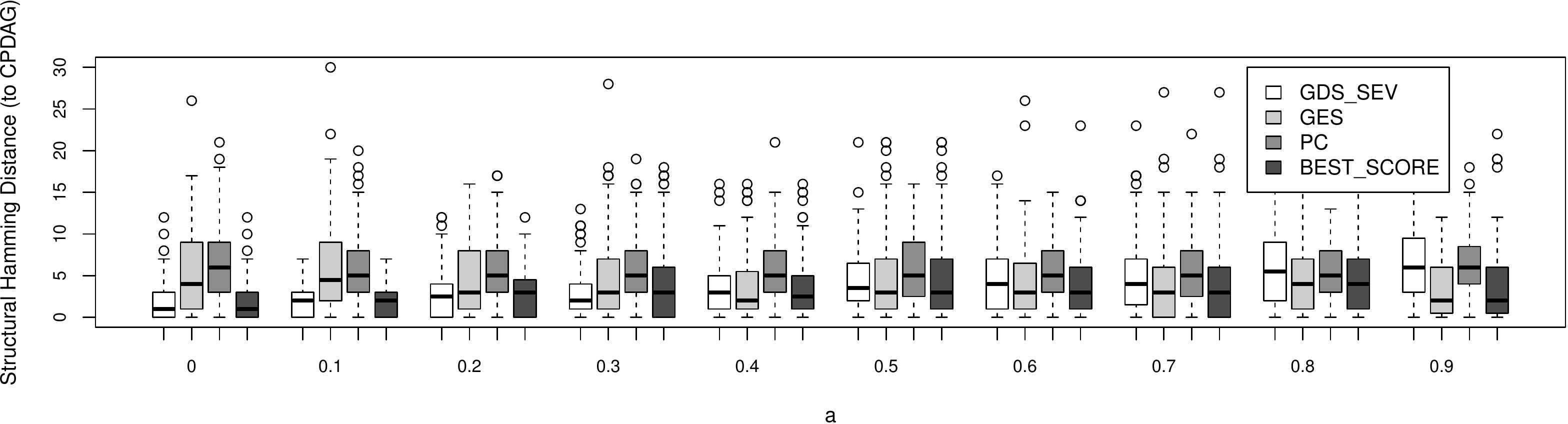}
\end{center}
\caption{Box plots for the structural Hamming distance of greedy directed acyclic graph search (white), greedy equivalence search (light grey), PC-algorithm (grey) and a best-score method (dark grey) to the true directed acyclic graph, \textsc{DAG}, (top) and to the true partially directed acyclic graph, \textsc{CPDAG}, (bottom). The graph shows various values of a measuring perturbation $a$ of equal error variances; only $a=0$ corresponds to equal error variances.}
\end{figure}

\subsection{Real data} \label{sec:real}
We now apply the greedy equivalence search and greedy directed acyclic graph search to seven data sets containing microarray data, described by \citet{Dettling2003} and \citet{Buhlmann2013}, and compare their \BIC scores. When greedy equivalence search obtains the better score, this indicates that the assumption of equal error variances is not justified. In Figure~\ref{fig:dev} we have seen that even then it might sometimes be useful to look at the greedy directed acyclic graph search solution. If, on the other hand, greedy directed acyclic graph search obtains a better score than greedy equivalence search, we prefer the solution obtained by greedy directed acyclic graph search, which furthermore is a graph rather than a Markov equivalence class.
To avoid a high-dimensional setting with $p>n$, we always chose the $0.8 n$ genes with the highest variance.
Table~\ref{tab:real} shows that in two out of the seven data sets, greedy directed acyclic graph search obtained a better score than greedy equivalence search.
\begin{table}[ht]
\begin{tabular}{lccccccc} 
 & Prostate &  Lymphoma    & Riboflavin &   Leukemia & Brain & Cancer  & Colon  \\ 
\GES    & $4095$ & $4560$ & ${2711}$ & ${ 5456}$ & $1411$ &${ 5891}$& $3224$ \\ 
\SEMSEV & $6057$ & $5404$       & $3236$ &       $5481$       & ${1343}$ &$6288$& ${3201}$ 
\end{tabular}
\label{tab:real}
\caption{\BIC scores of greedy equivalent search and greedy directed acyclic graph search on different type of microarray data; smaller is better. \textsc{ges}, greedy equivalence search; $\textsc{gds}_{\textsc{EEV}}$, greedy directed acyclic graph search with equal error variances.}
\end{table}
For the Colon example, greedy directed acyclic graph search proposes a directed acyclic graph with $192$ edges, greedy equivalence search a graph with $217$ edges. There are $91$ edges in both solutions, $61$ with the same orientation. The graphs therefore differ on roughly half of the edges.

\section*{Acknowledgement}
We thank R.~Tanase for fruitful discussions. The research leading to these results received funding from the European Union's Seventh Framework Programme. 

\appendix
\section*{Appendix}
\subsection*{Some lemmata} \label{sec:lem}
In the following two sections we consider different subsets of the set of variables $\B{X}$: to simplify notation we do not distinguish between indices and variables, since the context should clarify the meaning. This way, we can also speak of the parents $\PA[\G]{B}$ of a variable $B \in \B{X}$. We also consider sets of variables $\B{S} \subset \B{X}$ to be a single multivariate variable.

The following four statements are all plausible and their proofs mostly involve technicalities. The reader may skip to the next section 
and use the lemmata whenever needed.

\begin{lemma}\label{lem:cga}
Let $(A_1, \ldots, A_m) \sim \mathcal N\{(\mu_1, \ldots, \mu_m)^T, \Sigma\}$ with strictly positive definite $\Sigma$ and define 
$A_1^* = {A_1}_{\given (A_2, \ldots, A_m) = (a_2, \ldots, a_m)}$, in distribution.
Then $\var{A_1^*} \leq \var{A_1}$ for all $(a_2, \ldots, a_m) \in \R^{m-1}$.

\end{lemma}
We use the notation of conditional variables rather than conditional distributions to improve readability.
\begin{proof}
Let us decompose $\Sigma$ into
$$
\Sigma = \left(
\begin{array}{cc}
\sigma_1^2 & \Sigma_{12}^T\\
\Sigma_{12} & \Sigma_{22}
\end{array}
\right)
$$
with an $(m-1) \times 1$ vector $\Sigma_{12}$. Since $\Sigma_{22}^{-1}$ is positive definite,
$\var{A_1^*} = \sigma_1^2 - \Sigma_{12}^T \Sigma_{22}^{-1} \Sigma_{12} \leq \sigma_1^2$.
\end{proof}

\begin{lemma}\label{lem:cond}
\citep{Peters2011}
Let $Y,N ,Q$ and $R$ be random variables taking values in $\C{Y}, \C{N}, \C{Q}$ and $\C{R}$, respectively, whose joint distribution is absolutely continuous with respect
to some product measure; we denote the densities by $p_{Y,Q,R, N}(y,q,r,n)$. 
The variables $Q$ and $R$ can be multivariate.
Let $f : \C{Y} \times \C{Q} \times \C{N} \to \R$ be a measurable function.
If $N \independent (Y,Q,R)$ then for all $q \in \C{Q},r \in \C{R}$ with $p_{Q,R}(q,r) > 0$:
$
f(Y,Q,N){\given}_{Q=q,R=r} = f(Y{\given}_{Q=q,R=r},q,N)
$,
in distribution.
\end{lemma}

\begin{lemma}[\citet{Peters2011}]\label{lem:noi}
Let $\lawX$ be generated by a structural equation model as in~\eqref{eq:mm} with corresponding directed acyclic graph $\G$ and consider a variable $X \in \B{X}$. If $\B{S} \subseteq \ND[\G]{X}$ then $N_X \independent \B{S}$.
\end{lemma}

\begin{lemma} \label{lem:cmc}
Let $\law{\B{X}}$ be generated from a structural equation model as in~\eqref{eq:mm} with directed acyclic graph $\G$. Consider a variable $B \in \B{X}$ and one of its parents $A \in \PA[\G]{B}$. For all sets $\B{S}$ with $\PA[\G]{B}\setminus\{A\} \subseteq \B{S} \subseteq \ND[\G]{B} \setminus\{A\}$ we have
$B \notindependent A\,\mid\,\B{S}$. 
\end{lemma}
\begin{proof}
Define $Q = \PA[\G]{B}\setminus\{A\}$ such that we have $\B{S} = (Q,R)$ for some $R$. Using Lemma~\ref{lem:cond} we obtain:
$$
B{\mid}_{Q = q, R = r} = f(q) + \beta A{\mid}_{Q=q, R = r} + N_B\,, 
$$
in distribution, with $N_B \independent A{\mid}_{Q=q, R = r}$. But since $\beta \neq 0$,  
$
A{\mid}_{Q=q, R = r} \notindependent B{\mid}_{Q=q, R = r}\,.
$
\end{proof}

\subsection*{Proof of Theorem~\ref{thm:main}.} \label{sec:pro}
If we assumed faithfulness, we could recover the correct Markov equivalence class, which itself implies the existence of an $L$ and $Y$ shown in Remark~\ref{rem:proofidea} \citep[Theorem~2]{Chickering1995}. Since we are not assuming faithfulness, proving existence of a situation similar to that in Fig.~\ref{fig:ii2} requires more work. This part of the proof, due to not assuming faithfulness, is taken from \citet{Peters2011} and remains almost the same. The difference to \citet{Peters2011} is that we can prove causal minimality and need not assume it. New are also Lemmata~\ref{lem:cga} and \ref{lem:cmc}, as well as the proof's main argument given in the second part of case (ii).

\begin{proof}
We assume that there are two \SEMs as in equation~\eqref{eq:mm} that both induce $\law{\B{X}}$
, one with graph $\G$, the other with graph $\Gp$.
We will show that $\G = \Gp$.
Since directed acyclic graphs do not contain any cycles, we always find nodes that have no descendants. To see this start a directed path at some node; after at most $\#\B{X}-1$ steps we reach a node without a child. Eliminating such a node from the graph leads to a directed acyclic graph, again; we can discard further nodes without children in the new graph. We repeat this process for all nodes that have no children in both $\G$ and $\Gp$ and have the same parents in both graphs. If we end up with no nodes left, the two graphs are identical and the result is proved. Otherwise, we end up with a smaller set of variables that we again call $\X$, two smaller graphs that we again call $\G$ and $\Gp$ and a node $L$ that has no children in $\G$ and either $\PA[\G]{L}\neq \PA[\Gp]{L}$ or $\CH[\Gp]{L} \neq \emptyset$.
We will show that this leads to a contradiction. Importantly, because of the Markov property of the distribution with respect to $\G$, all other nodes are independent of $L$ given $\PA[\G]{L}$:
\begin{equation} \label{eq_xind}
L \independent \B{X} \setminus (\PA[\G]{L} \cup \{L\}) \,\given\, \PA[\G]{L}\,.
\end{equation}

To make the arguments easier to understand, we introduce the following notation, see also Fig.~\ref{fig:wcp}. We partition $\G$-parents of $L$ into $\B{Y}, \B{Z}$ and $\B{W}$. Here, $\B{Z}$ are also $\Gp$-parents of $L$, $\B{Y}$ are $\Gp$-children of $L$ and $\B{W}$ are not adjacent to $L$ in $\Gp$. Let $\B{D}$ be the $\Gp$-parents of $L$ that are not adjacent to $L$ in $\G$ and by $\B{E}$ the $\Gp$-children of $L$ that are not adjacent to $L$ in $\G$. 
\begin{figure}[t]
  \begin{minipage}[t]{0.48\columnwidth}
    \begin{center}
      \begin{tikzpicture}[xscale=1.08,yscale=0.84, line width=0.5pt, minimum size=0.58cm, inner sep=0.3mm, shorten >=1pt, shorten <=1pt]
	\small
        \draw (0,0) node(x) [circle, draw] {$L$};
        \draw (-1.3,1.2) node(w) [circle, draw] {$\B{W}$};
        \draw (0,1.2) node(y) [circle, draw] {$\B{Y}$};
        \draw (1.3,1.2) node(z) [circle, draw] {$\B{Z}$};
        \draw[-arcsq] (z) -- (x);
        \draw[-arcsq] (y) -- (x);
        \draw[-arcsq] (w) -- (x);
      \end{tikzpicture}\\
      \small part of $\G$
    \end{center}
  \end{minipage}
  \hspace{0.02\columnwidth}
  \begin{minipage}[t]{0.48\columnwidth}
    \begin{center}
      \begin{tikzpicture}[xscale=1.08,yscale=0.84, line width=0.5pt, minimum size=0.58cm, inner sep=0.3mm, shorten >=1pt, shorten <=1pt]
	\small
        \draw (0,-0.6) node(x) [circle, draw] {$L$};
        \draw (-1.3,0) node(d) [circle, draw] {$\B{D}$};
        \draw (1.3,0) node(z) [circle, draw] {$\B{Z}$};
        \draw (1.3,-1.2) node(e) [circle, draw] {$\B{E}$};
        \draw (-1.3,-1.2) node(y) [circle, draw] {$\B{Y}$};
        \draw[-arcsq] (z) -- (x);
        \draw[-arcsq] (d) -- (x);
        \draw[-arcsq] (x) -- (y);
        \draw[-arcsq] (x) -- (e);
      \end{tikzpicture}\\
      \small part of $\Gp$
    \end{center}
  \end{minipage}
  \caption{Nodes adjacent to $L$ in $\G$ and $\Gp$ \label{fig:wcp}}
\end{figure}
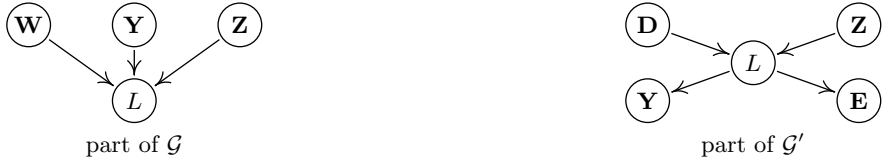
Thus:
$\PA[\G]{L} = \B{Y} \cup \B{Z} \cup \B{W}$, $\CH[\G]{L} = \emptyset,$
$\PA[\Gp]{L} = \B{Z} \cup \B{D}$, $\CH[\Gp]{L} = \B{Y} \cup \B{E}$.
Consider $\B{T} = \B{W} \cup \B{Y}$. We distinguish two cases.


Case (i):
$\B{T} = \emptyset$.
Then there must be a node $D \in \B{D}$ or a node $E \in \B{E}$, otherwise $L$ would have been discarded.
If there is a $D \in \B{D}$ then \eref{eq_xind} implies
$L \independent D \given \B{S}$
for $\B{S} = \B{Z} \cup \B{D} \setminus \{D\}$, 
which contradicts Lemma~\ref{lem:cmc} applied to $\Gp$.
If $\B{D}=\emptyset$ and there is $E \in \B{E}$ then 
$E \independent L \given \B{S}$
holds for $\B{S}=\B{Z} \cup \PA[\Gp]{E} \setminus \{L\}$, which also contradicts Lemma \ref{lem:cmc}; to avoid cycles it is necessary that $\B{Z}\subseteq \ND[\Gp]{E}$ .

Case (ii):
$\B{T} \ne \emptyset$.
Then $\B{T}$ contains a $\Gp$-youngest node with the property that there is no
directed $\Gp$-path from this node to any other node in $\B{T}$. 
This node may not be unique.

Suppose that $W \in \B{W}$ is such a youngest node. Consider the directed acyclic graph $\tilde \Gp$ that equals $\Gp$ with additional edges $Y \rightarrow W$ and $W' \rightarrow W$ for all $Y \in \B{Y}$ and $W' \in \B{W} \setminus \{W\}$. In $\tilde \Gp$, $L$ and $W$ are not adjacent. Thus we find a set $\B{\tilde S}$ such that $\B{\tilde S}$ $d$-separates $L$ and $W$ in $\tilde \Gp$; indeed, one can take 
$\B{\tilde S} = \PA[\tilde \Gp]{L}$ if $W \notin \DE[\tilde \Gp]{L}$ and $\B{\tilde S} = \PA[\tilde \Gp]{W}$ if $L \notin \DE[\tilde \Gp]{W}$.
Then $\B{S}=\B{\tilde S} \cup \{\B{Y},\B{Z},\B{W}\setminus \{W\}\}$ $d$-separates $L$ and $W$ in $\tilde \Gp$.

We now prove this claim. All $Y \in \B{Y}$ are already in $\B{\tilde S}$ in order to block $L \rightarrow Y \rightarrow W$. Suppose there is a $\tilde \G'$-path that is blocked by $\B{\tilde S}$ and unblocked if we add $Z$ and $W'$ nodes to $\B{\tilde S}$. 
How can we unblock a path by including more nodes? The path $L \cdots V_1 \cdots U_1 \cdots W$, see Fig.~\ref{fig:indeed}, must contain a collider $V_1$ that is an ancestor of a $Z$ with $V_1, \ldots, V_m, Z \notin \B{\tilde S}$ and corresponding nodes $U_i$ for a $W'$ node. 
Choose $V_1$ and $U_1$ on the given path so close to each other such that there is no such collider in between. If there is no $V_1$, choose $U_1$ close to $L$, if there is no $U_1$, choose $V_1$ close to $W$. Now the path $L \leftarrow Z \cdots V_1 \cdots U_1 \cdots W' \rightarrow W$ is unblocked given $\B{\tilde S}$, which contradicts the fact that $\B{\tilde S}$ $d$-separates $L$ and $W$. This ends the claim's proof.

The set $\B{S}$ $d$-separates $L$ and $W$ also in $\Gp$ because $\Gp$ contains less paths. We have
$
L \independent W \,\mid\,\B{S}
$
which contradicts Lemma \ref{lem:cmc} applied to $\G$. 
Summarizing, $W \in \B{W}$ cannot be the $\Gp$-youngest node.
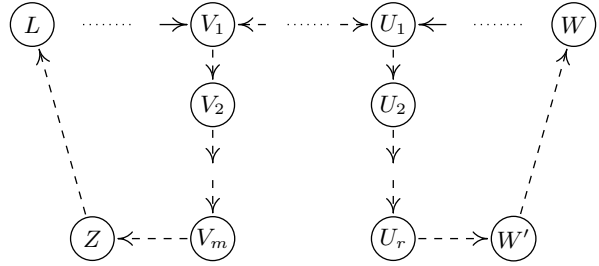
\begin{figure}[ht]
    \begin{center}
      \begin{tikzpicture}[xscale=0.8, yscale=0.89, line width=0.5pt, minimum size=0.58cm, inner sep=0.3mm, shorten >=1pt, shorten <=1pt]
	\small
        \draw (-3,0) node(x) [circle, draw] {$L$};
        \draw (6,0) node(w) [circle, draw] {$W$};
        \draw (0,0) node(v1) [circle, draw] {$V_1$};
        \draw (0,-1.2) node(v2) [circle, draw] {$V_2$};
        \draw (0,-3.2) node(vm) [circle, draw] {$V_m$};
        \draw (-2,-3.2) node(z) [circle, draw] {$Z$};
        \draw (3,0) node(u1) [circle, draw] {$U_1$};
        \draw (3,-1.2) node(u2) [circle, draw] {$U_2$};
        \draw (3,-3.2) node(us) [circle, draw] {$U_r$};
        \draw (5,-3.2) node(wp) [circle, draw] {$W'$};
        \draw (-1.3,0) node(x3) [circle, draw, white] {$L$};
        \draw (1.3,0) node(x4) [circle, draw, white] {$L$};
        \draw (1.7,0) node(x5) [circle, draw, white] {$L$};
        \draw (4.3,0) node(x6) [circle, draw, white] {$L$};
        \draw (0,-2.45) node(v3) [circle, draw, white] {$L$};
        \draw (0,-1.95) node(v4) [circle, draw, white] {$L$};
        \draw (3,-2.45) node(u3) [circle, draw, white] {$L$};
        \draw (3,-1.95) node(u4) [circle, draw, white] {$L$};
        \draw (1.7,0) node(x5) [circle, draw, white] {$L$};
        \draw (4.3,0) node(x6) [circle, draw, white] {$L$};
        \draw[-arcsq, dashed] (z) -- (x);
        \draw[-arcsq] (x3) -- (v1);
        \draw[-arcsq, dashed] (x4) -- (v1);
        \draw[-arcsq, dashed] (wp) -- (w);
        \draw[-arcsq, dashed] (v1) -- (v2);
        \draw[-arcsq, dashed] (v2) -- (v3);
        \draw[-arcsq, dashed] (v4) -- (vm);
        \draw[-arcsq, dashed] (vm) -- (z);
        \draw[-arcsq, dashed] (x5) -- (u1);
        \draw[-arcsq] (x6) -- (u1);
        \draw[-arcsq, dashed] (u1) -- (u2);
        \draw[-arcsq, dashed] (u2) -- (u3);
        \draw[-arcsq, dashed] (u4) -- (us);
        \draw[-arcsq, dashed] (us) -- (wp);
	\draw[dotted] (-2.2,0) -- (-1.3,0);
	\draw[dotted] (1.2,0) -- (2.0,0);
	\draw[dotted] (4.3,0) -- (5.2,0);
      \end{tikzpicture}
    \end{center}
  \caption{Assume the path $L \cdots V_1 \cdots U_1 \cdots W$ is blocked by $\B{\tilde S}$, but unblocked if we include $Z$ and $W'$. Then the dashed path is unblocked given $\B{\tilde S}$. \label{fig:indeed}}
\end{figure}

Therefore,
the $\Gp$-youngest node in $\B{T}$ must be some $Y \in \B{Y}$.
It holds that
\begin{equation} \label{eq:gs}
\sigma_{\G}^2  = \sigma_{\Gp}^2 = \min_{X \in \B{X}} \var{X} = \sigma^2\,.
\end{equation}
We define
$\B{S}=\PA[\G]{L} \setminus \{Y\} \cup \B{D}$.
Clearly, $\B{S} \subseteq \ND[\G]{L}$ since $L$ does not have any descendants in $\G$. 
Define $Q = \PA[\G]{L} \setminus \{Y\}$ and
take any 
$s=(q, d)$. Define
$L^* = L{\given}_{\B{S} = s}$, in distribution,  and  $Y^* = Y{\given}_{\B{S} = s}$, in distribution.
Then, from $\G$ and using Lemma~\ref{lem:cond} we find
$L^* = f_L(q, Y^*) + N_L = f(q) + \beta \cdot Y^* + N_L$, in distribution, with $N_L \independent Y{\given}_{\B{S} = s}$.
The independence holds because $\B{S} \subseteq \ND[\G]{L}$. Then, we have
\begin{equation} \label{equ:g}
\var{L^*} = \beta^2  \var{Y^*} + \sigma^2 > \sigma^2\,.
\end{equation}
Since $\PA[\Gp]{L} \subseteq \B{S}$ we find from $\Gp$ and Lemma~\ref{lem:cga} that
\begin{equation} \label{equ:gp}
\var{L^*} \leq \sigma^2\,.
\end{equation}
since $\det \{\mathrm{cov} (\B{X})\} \neq 0$.
Equations~\eqref{equ:g} and \eqref{equ:gp} contradict each other.

To prove Remark~\ref{rem:sca}, replace $\var{X}$ by $\var{X} / \alpha_X$ in \eqref{eq:gs} and $\sigma^2$ by $\sigma^2 \alpha_X$ in~\eqref{equ:g} and \eqref{equ:gp}.
\end{proof}




\bibliographystyle{plainnat}
\bibliography{bibliography}



\end{document}